\documentclass[a4paper,conference]{IEEEtran}
\usepackage{ifpdf}
\usepackage{cite}
\usepackage[pdftex]{graphicx}
\usepackage{amsmath,amssymb,amsthm,amsfonts}
\usepackage{algorithm}
\usepackage{algorithmic}
\usepackage{array}
\usepackage{url}
\usepackage{color}

\newtheorem{theorem}{Theorem}

\newtheorem{assumption}[theorem]{Assumption}

\newtheorem{corollary}[theorem]{Corollary}

\DeclareMathOperator*{\argmin}{arg\,min}

\newcommand{\algoref}[1]{Algorithm~\ref{#1}}
\newcommand{\figref}[1]{Figure~\ref{#1}}
\newcommand{\thref}[1]{Theorem~\ref{#1}}

\newcommand{\hA}{\hat{A}}
\newcommand{\hB}{\hat{B}}
\newcommand{\hC}{\hat{C}}
\newcommand{\hD}{\hat{D}}

\begin{document}

\title{An Explicit Rate Bound for Over-Relaxed ADMM}

\author{
\IEEEauthorblockN{Guilherme Fran\c{c}a}
\IEEEauthorblockA{
francag@bc.edu}
\and
\IEEEauthorblockN{Jos\'e Bento}
\IEEEauthorblockA{
jose.bento@bc.edu}
}

\maketitle

\begin{abstract}
The framework of Integral Quadratic Constraints of Lessard et al. (2014) 
reduces the computation
of upper bounds on the convergence 
rate of several optimization algorithms to 
semi-definite programming (SDP). Followup work by
Nishihara et al. (2015) applies this technique to the entire family
of over-relaxed 
Alternating Direction Method of Multipliers (ADMM). 
Unfortunately, they only provide an
explicit error bound for sufficiently large values
of some of the parameters of the problem, leaving the computation
for the general case as a numerical optimization problem. 
In this paper we provide an exact analytical solution to this SDP
and obtain a general and explicit  
upper bound on the convergence rate of the entire family of
over-relaxed ADMM. 
Furthermore, we demonstrate
that it is not possible to extract from this SDP a general bound
better than ours.
We end with a few numerical illustrations of our
result and a comparison between the
convergence rate we obtain for ADMM with known
convergence rates for Gradient Descent (GD).
\end{abstract}

\IEEEpeerreviewmaketitle

\section{Introduction}

Consider the optimization problem
\begin{equation} \label{minimize}
\begin{split}
  \text{minimize} & \quad f(x) + g(z) \\
  \text{subject to} & \quad Ax + Bz = c 
\end{split}
\end{equation}
where $x\in \mathbb{R}^p$, $z\in \mathbb{R}^q$, $A\in \mathbb{R}^{r \times p}$,
$B\in \mathbb{R}^{r \times q}$,
and $c\in \mathbb{R}^{r}$ under the following additional 
assumption, which we assume throughout the paper.
\begin{assumption}\label{assumption}
\hspace{0.001cm}\\
\vspace{-0.4cm}
\begin{enumerate}
\item The functions $f$ and $g$ are convex, closed and proper;
\item Let $S_d(m,L)$ be the set of 
functions $h:\mathbb{R}^d \to \mathbb{R}\cup\{+\infty\}$ such that
\begin{equation*}
m\|x-y\|^2\le\left(\nabla h(x)-\nabla h(y)\right)^T(x-y)\le L\| x-y\|^2
\end{equation*}
for all $x,y\in\mathbb{R}^d$ where $0 < m \le L < \infty$;
We assume that 
$f \in S_p(m,L)$, in other words, $f$ is strongly
convex and $\nabla f$ is Lipschitz continuous; and that 
$g \in S_q(0,\infty)$;
\item $A$ is invertible and $B$ has full column rank.
\end{enumerate}
\end{assumption}

In this paper we give an explicit convergence rate bound for
a family of optimization schemes known as over-relaxed ADMM
when applied to the optimization problem \eqref{minimize}.
This family is parametrized by $\alpha > 0$ and $\rho > 0$ and when 
applied to \eqref{minimize} takes the form in Algorithm \ref{ADMM}.
A classical choice of parameters is $\alpha = 1$ and $\rho = 1$.
\begin{algorithm}
\caption{Family of Over-Relaxed ADMM schemes (parameters $\rho$,  $\alpha$)} \label{ADMM}
\begin{algorithmic}[1]
  \STATE {\bfseries Input:} $f$, $g$, $A$, $B$, $c$; 
  \STATE Initialize $x_0, z_0, u_0$
  \REPEAT
  \STATE $x_{t+1} = \argmin_x f(x) + \tfrac{\rho}{2} \|Ax + Bz_t - c + u_t\|^2$
  \STATE $z_{t+1} = \argmin_z g(z) + \tfrac{\rho}{2} \|\alpha Ax_{t+1} - (1-\alpha)Bz_t  + Bz - \alpha c + u_t\|^2$
  \STATE $u_{t+1} = u_t + \alpha Ax_{t+1} - (1-\alpha)Bz_t + Bz_{t+1} - \alpha c$
  \UNTIL{stop criterion}
\end{algorithmic}
\end{algorithm}
Several works
have computed specific rate bounds for ADMM
under specific different regimes but a recent work by \cite{Lessard}
allowed \cite{Jordan} to reduce the analysis of this entire family
of solvers to finding solutions for a semi-definite programming problem. 
This SDP has multiple solutions and different solutions
give different bounds on the convergence rate of ADMM,
some better than others.
In \cite{Jordan} they analyze this SDP
numerically and also give one feasible solution to this SDP
when $\kappa = (L/m) \kappa^2_A$ is sufficiently large, $\kappa_A$ 
being the condition number of $A$.
They further show, via a lower bound, that it is not possible
to extract from this SDP a rate that is much better than 
the rate associated with their solution for large $\kappa$.

An important problem remains open that we solve in this paper.
Can we find a general explicit expression for the best\footnote{``Best'' 
in the sense that it gives the smallest rate bound.}
solution of this SDP?
The answer is yes.
As we explain later, our finding has both theoretical and practical interest.

\section{Main results}

We start by recalling the main result of \cite{Jordan} which is
the starting point of our work.
Based on the framework proposed in \cite{Lessard},  
it was later shown \cite{Jordan} that the iterative scheme of 
\algoref{ADMM} can be written as a dynamical system involving the matrices
\vspace{-.25cm}
\begin{equation}
\begin{aligned}
\hA &= \begin{bmatrix} 1 & \alpha - 1 \\ 0 & 0 \end{bmatrix},
& \hB &= \begin{bmatrix} \alpha & - 1 \\ 0 & -1 \end{bmatrix}, \\
\hC_1 &= \begin{bmatrix} -1 & -1 \\ 0 & 0 \end{bmatrix}, 
& \hC_2 &= \begin{bmatrix} 1 & \alpha - 1 \\ 0 & 0 \end{bmatrix}, \\
\hD_1 &= \begin{bmatrix} -1 & 0 \\ 1 & 0 \end{bmatrix},
& \hD_2 &= \begin{bmatrix} \alpha & -1 \\ 0 & 1 \end{bmatrix},
\end{aligned}
\end{equation}
and the constants
\vspace{-.25cm}
\begin{subequations}\label{parameters}
\begin{align}
\hat{m} &= \dfrac{m}{\sigma_1^2(A)}, & \hat{L} &= \dfrac{L}{\sigma_p^2(A)}, 
\label{hats} \\
\rho_0 &= \rho (\hat{m}\hat{L})^{-1/2}, & \kappa &= \kappa_f \kappa_A^2.
\label{params}
\end{align}
\end{subequations}
Above, $\kappa_f = L/m$, $\sigma_1(A)$ and $\sigma_p(A)$ denote the largest
and smallest singular
value of the matrix $A$, respectivelly, 
and $\kappa_A = \sigma_1(A)/\sigma_p(A)$ is the condition number of $A$.
Throughout the paper $\kappa_M$ denotes the condition 
number of matrix $M$.

The stability of this dynamical system is then related to the convergence rate of \algoref{ADMM} which in turn involves numerically solving
a $4\times4$ semidefinite program as stated in following theorem.

\begin{theorem}[See \cite{Jordan}] \label{jordan_theo}
Let the sequences 
$\left\{ x_t \right\}$,
$\left\{ z_t \right\}$, and
$\left\{ u_t \right\}$ evolve according to \algoref{ADMM} with
step size $\rho > 0$ and relaxation parameter
$\alpha > 0$. 
Let $\varphi_t = \left[ z_t, u_t\right]^{T}$
and $\varphi_{*}$ be a fixed point of the algorithm.
Fix $0 < \tau < 1$ and suppose there is a $2\times 2$ 
matrix $P \succ 0$ and constants $\lambda_1,\lambda_2 \ge 0$
such that
\begin{multline}\label{semidefinite}
\begin{bmatrix}
\hA^T P \hA-\tau^2 P & \hA^T P \hB \\ \hB^T P \hA & \hB^T P \hB 
\end{bmatrix} + \\
\begin{bmatrix} \hC_1 & \hD_1 \\ \hC_2 & \hD_2 \end{bmatrix}^{T}
\begin{bmatrix} \lambda_1 M_1 & 0 \\ 0 & \lambda_2 M_2 \end{bmatrix}
\begin{bmatrix} \hC_1 & \hD_1 \\ \hC_2 & \hD_2 \end{bmatrix} \preceq 0
\end{multline}
where
\begin{align}
M_1 &= \begin{bmatrix} 
-2\rho_0^{-2} & \rho_0^{-1}(\kappa^{1/2}+\kappa^{-1/2})  \\
\rho_0^{-1}(\kappa^{1/2}+\kappa^{-1/2}) & -2
\end{bmatrix}, \\
M_2 &= \begin{bmatrix} 0 & 1 \\ 1 & 0 \end{bmatrix}.
\end{align}
Then for all $t\ge 0$ we have
\begin{equation}\label{bound_admm}
\| \varphi_t - \varphi_* \| \le \kappa_B \sqrt{\kappa_P} \, \tau^t
\| \varphi_0 - \varphi_* \|.
\end{equation}
\end{theorem}
It is not hard to see that any fixed point of Algorithm~\ref{ADMM} 
satisfies the KKT conditions for problem \eqref{minimize} which, due to 
convexity, make it a solution to problem \eqref{minimize}.
Moreover, since $A$ is non-singular, by step $6$ in Algorithm~\ref{ADMM}
the rate bound $\tau$ also bounds $\|[x_t,z_t,u_t] - [x_*,z_*,u_*]\|$.

As already pointed out in \cite{Jordan}, the weakness of \thref{jordan_theo} is
that $\tau$ is not explicitly given as a function of the parameters involved
in the problem, namely $\kappa$, $\rho_0$, and $\alpha$. The factor
$\kappa_P$ in \eqref{bound_admm} is also not explicitly given. Therefore,
for given values of these parameters one must perform
a numerical search to find the minimal $\tau$ such 
that \eqref{semidefinite} is feasible.
This in turn implies, for example, that to optimally tune ADMM
using this bound one might have perform this numerical search multiple times
scanning the parameter space $(\alpha, \rho_0)$.

While from a practical point of view this may be enough for many purposes,
this procedure can certainly introduce delays if, for example, \eqref{semidefinite} 
is used in an adaptive scheme where after every few iterations we estimate a local value for $\kappa$ and then re-optimize $\alpha$ and $\rho$.
Therefore, it is desirable to have an explicit expression for the smallest
$\tau$ that \thref{jordan_theo} can provide, from which the optimal values of the
parameters follow. This expression is also desirable from a 
theoretical point of view. Our main goal in this paper is to
complete the work initiated in \cite{Jordan}, thus providing an explicit
formula for the rate bound that the method proposed in \cite{Lessard}
can provide when applied to over-relaxed ADMM.

Two of the most explicit bound rates that resemble the bound we give in this section
are the ones found in \cite{boyd_almost_bound} and \cite{wei_almost_bound}.
The authors in \cite{boyd_almost_bound} analyze the Douglas-Rachford splitting method, a scheme different but related to the one we analyze in this paper, for a problem similar to \eqref{minimize}, and give
a rate bound of $1 - \frac{\alpha}{1 + \sqrt{\kappa_f}}$ where $\alpha$ is a step size and $\kappa_f = L/m$ where $L$ and $m$ bound the curvature of the
objective function in the same sense as in Assumption \ref{assumption}.
The authors in \cite{wei_almost_bound} apply ADMM with $\alpha =1$ and
$\rho_0 = 1$ to the same problem as we do and give a rate bound of $1 -
\frac{1}{\sqrt{\kappa}} + O(\frac{1}{\kappa})$, where
$\kappa = \kappa_f \kappa^2_A$.

We now state and prove our main results.
Throughout the paper we often make use of the function
\begin{equation}
\chi(x) = \max(x, x^{-1}) \ge 1 \quad \mbox{for $x \in \mathbb{R} > 0$}.
\end{equation}

\begin{theorem} \label{explicit_solution}
For $0 < \alpha \leq 2$, $\kappa > 1$ and $\rho_0>0$, the following is an explicit feasible point 
of \eqref{semidefinite} with $\lambda_1,\lambda_2 \geq  0$, $P \succ 0$ and $0<
\tau < 1$:
\begin{align}
P &= \begin{pmatrix} 1 & \xi \\
                     \xi & 1 \end{pmatrix}, \qquad
\xi = -1 + \dfrac{\alpha (\chi(\rho_0)\sqrt{\kappa}-1)}{
1-\alpha+\chi(\rho_0)\sqrt{\kappa}}, 
\label{P_sol}\\
\lambda_1 &= \dfrac{\alpha \rho_0 \sqrt{\kappa}
\left(1-\alpha+\chi(\rho_0)\sqrt{\kappa}\right)}{
(\kappa-1)\left(1+\chi(\rho_0)\sqrt{\kappa}\right)}, 
\label{l1_sol} \\
\lambda_2 &= 1 + \xi, \label{l2_sol}
\end{align}
with
\begin{equation} \label{tau_sol}
\tau = 1 - \dfrac{\alpha}{1+\chi(\rho_0)\sqrt{\kappa}}.
\end{equation}
\end{theorem}
\begin{proof}
First notice that since $\kappa > 1$ and $\chi(\rho_0) \geq 1$ we have
$\lambda_1,\lambda_2 \geq 0$ for the allowed range of parameters.
Second, notice that the eigenvalues of $P$ are $1+\xi$ and $1-\xi$,
and since $\xi > -1$ we have $P \succ 0$.
Finally, consider the full matrix in the left hand side of \eqref{semidefinite} and
let $D_n$ denote an $n$th principal minor.
We will show through direct computation that $(-1)^n D_n \ge 0$ for all 
principal minors, which proves our claim.

Replacing \eqref{P_sol}--\eqref{tau_sol} into \eqref{semidefinite}
we have the matrix
shown in equation \eqref{big_matrix}.
Note that it has vanishing determinant $D_0 = 0$.
Let $J\subseteq \{1,2,3,4\}$ and denote $D_n^J$ the $n$th principal 
minor obtained by deleting the rows and columns with indices in 
$J$.

{
\begin{figure*}[!t]
\normalsize
\def\arraystretch{1.8}
\setlength{\arrayrulewidth}{.18pt}
\begin{equation}\label{big_matrix}
\left(
\begin{array}{c|c|c|c}
1-\tau^2-\dfrac{2\lambda_1}{\rho_0^2} &
\alpha-1-\xi\tau^2 - \dfrac{2\lambda_1}{\rho_0^2} & 
\alpha-\dfrac{2\sqrt{\kappa}+\rho_0(1+\kappa)}{\rho_0^2\sqrt{\kappa}}\lambda_1&
0 \\
\alpha-1-\xi\tau^2 - \dfrac{2\lambda_1}{\rho_0^2} & 
(\alpha-1)^2  - \tau^2 - \dfrac{2\lambda_1}{\rho_0^2} &
\alpha(\alpha-1) - 
\dfrac{2\sqrt{\kappa}+\rho_0(1+\kappa)}{\rho_0^2\sqrt{\kappa}}\lambda_1 &
0 \\
\alpha-\dfrac{2\sqrt{\kappa}+\rho_0(1+\kappa)}{\rho_0^2\sqrt{\kappa}}\lambda_1&
\alpha(\alpha-1) - 
\dfrac{2\sqrt{\kappa}+\rho_0(1+\kappa)}{\rho_0^2\sqrt{\kappa}}\lambda_1 &
\alpha^2 - 
\dfrac{2\rho_0^2\sqrt{\kappa} + 2\sqrt{\kappa} + 
2\rho_0(1+\kappa)}{\rho_0^2\sqrt{\kappa}} \lambda_1 &
0 \\
0 & 0 & 0 & 0
\end{array}
\right)
\end{equation}
\begin{subequations}\label{minors1}
\begin{align}
D_2^{\{4,3\}} &= \dfrac{2\alpha^2(2-\alpha)(\rho_0^2-1)\sqrt{\kappa} 
(1-\alpha+\rho_0\sqrt{\kappa})}{(\kappa-1)\rho_0(1+\rho_0\sqrt{\kappa})^3} 
\label{m1} \\
D_2^{\{4,1\}} &= D_2^{(4,3)} \cdot (1+\rho_0\sqrt{\kappa})^2 \label{m2} \\
D_1^{\{4,3,2\}} &= \alpha\cdot\dfrac{2(\alpha-1)\sqrt{\kappa}+(\alpha-2)
(1+\kappa)\rho_0-2\rho_0^2\sqrt{\kappa}}{
(\kappa-1)\rho_0(1+\rho_0\sqrt{\kappa})^2} \label{m3} \\
D_1^{\{4,2,1\}} &= D_1^{\{4,3,2\}} \cdot (1+\rho_0\sqrt{\kappa})^2 \label{m4} \\
D_1^{\{4,3,1\}} &= \alpha\sqrt{\kappa}\cdot\dfrac{
2(\alpha-1)+\rho_0\big\{ 2(\alpha-2)\sqrt{\kappa}
+\rho_0\big( 2+2\alpha(\kappa-1)-4\kappa+
(\alpha-2)(\kappa-1)\rho_0\sqrt{\kappa}\big)
\big\}
}{
(\kappa-1)\rho_0(1+\rho_0\sqrt{\kappa})^2} \label{m5}
\end{align}
\end{subequations}
\hrulefill
\vspace*{4pt}
\vspace{-0.5cm}
\end{figure*}
}

We consider the case $\rho_0 \ge 1$ first.
The only \emph{nonvanishing} principal minors are shown in 
equation \eqref{minors1}.
We obviously have $\eqref{m1}, \eqref{m2} \ge 0$ for the allowed range 
of parameters.
For \eqref{m3} and \eqref{m4} we need to show that
the concave
$2$nd order polynomial
$w(\tilde{\kappa}) = 
2(\alpha-1)\tilde{\kappa}+(\alpha-2)(1+\tilde{\kappa}^2)\rho_0 - 
2\rho_0^2\tilde{\kappa}$ is non-positive for $\tilde{\kappa} \equiv \sqrt{\kappa} > 1$. To do this it suffices to show that the function and its first
derivative are non-positive for $\tilde{\kappa} > 1$. We have
$\partial_{\tilde{\kappa}}w(1) = w(1) = 2(1+\rho_0)(\alpha-1-\rho_0) \le 0$.
Therefore, $w(\tilde{\kappa}) \le 0$ for $\tilde{\kappa} > 1$ implying that
\mbox{$\eqref{m3}, \eqref{m4} \le 0$}, as required.
Analogously, for \eqref{m5} we only need to show that, for the allowed range 
of parameters, the $3$rd order degree
polynomial in $\tilde{\kappa} \equiv \sqrt{\kappa}$,
\vspace{-.18cm}
\begin{align}
w(\tilde{\kappa}) &= 2(\alpha-1)+\rho_0\Big\{ 2(\alpha-2)\tilde{\kappa} 
+\rho_0\big\{
2+2\alpha(\tilde{\kappa}^2-1) \nonumber \\ &
 -4\tilde{\kappa}^2+(\alpha-2)(\tilde{\kappa}^2-1)\rho_0\tilde{\kappa}
\big\}
\Big\},
\end{align}
which is the numerator in the fraction \eqref{m5}, is non-positive for $\tilde{\kappa} > 1$.
To do this it suffices to show that the zeroth, first and second 
derivatives are non-positive.
We have $w(1) = 2(\alpha-1-\rho_0)(1+\rho_0) \le 0$,
$\partial_{\tilde{\kappa}}w(1) = 2(\alpha-2)\rho_0(1+\rho_0)^2 \le 0$,
and 
\mbox{$\partial_{\tilde{\kappa}}^2w(1) = 
2(\alpha-2)\rho_0^2(2+3\rho_0) \le 0$}.
This implies that $w(\tilde{\kappa}) \le 0$ for $\tilde{\kappa} > 1$ and
consequently $\eqref{m5} \le 0$. This concludes the proof for 
$\rho_0 \ge 1$.

For $\rho_0 < 1$ the analogous
expressions to \eqref{minors1} are slightly different but the
previous argument holds in exactly the same manner, 
thus we omit the details.
\end{proof}

In the following corollary, we allow $\kappa =1$ but $0 < \alpha < 2$.
It gives an explicit bound on the convergence rate of over-relaxed ADMM.
\begin{corollary}\label{upper_bound_admm}
Consider the sequences 
$\left\{ x_t \right\}$,
$\left\{ z_t \right\}$, and
$\left\{ u_t \right\}$, updated according to \algoref{ADMM} with
step size $\rho > 0$, relaxation parameter
$0 < \alpha < 2$ and for a problem with $\kappa \geq 1$. 
Let $\varphi_t = \left[z_t, u_t\right]^{T}$
and $\varphi_{*}$ be a fixed point. Then the convergence rate of
over-relaxed ADMM obeys the following upper bound:
\begin{equation}\label{convergence}
\| \varphi_t - \varphi_* \| \le \kappa_B \sqrt{\chi(\eta)} \, \tau^t
\| \varphi_0 - \varphi_* \|
\end{equation}
with $\tau$ explicitly given by the formula \eqref{tau_sol} and
\begin{equation}\label{eta_rate}
\eta = \dfrac{\alpha}{2-\alpha} \cdot
\dfrac{\chi(\rho_0)\sqrt{\kappa}-1}{\chi(\rho_0)\sqrt{\kappa}+1}.
\end{equation}
\end{corollary}
\begin{proof}
The proof for $\kappa > 1$ follows directly from \thref{jordan_theo} and
Theorem \ref{explicit_solution}. Indeed, all that we need to do is
to compute $\kappa_P$ in \eqref{bound_admm} for $P$ as in Theorem \ref{explicit_solution}. The two eigenvalues of $P$ are $1 - \xi$ and
$1 + \xi$ and the ratio of the largest to the smallest is 
precisely $\chi(\eta)$, where $\eta$ is given in equation \eqref{eta_rate}.

For  $\kappa = 1$ the proof follows by continuity. 
First notice that, from one iteration to the next in Algorithm \ref{ADMM},
$(x_{t+1},z_{t+1},u_{t+1})$ is a continuous function of $(x_{t},z_{t},u_{t},A)$
in a neighborhood of an invertible $A$, if we assume everything 
else fixed (this can be derived from the properties of proximal operators, c.f. \cite{ProximalBoyd}). 
Therefore, by the continuity of the composition of continuous functions, 
and assuming only $A$ is free and everything else is fixed,
$\| \varphi_t - \varphi_* \| = F(A)$ for some function that is continuous around a neighborhood of an invertible $A$.
Now, add a small perturbation $\delta A$ to $A$ such that $\kappa_A > 1$.
This perturbation makes $\kappa > 1$ and by the first part of this proof we can write that
$F(A + \delta A) \leq \kappa_B \sqrt{\chi(\eta + \delta \eta)} (\tau + \delta 
\tau)^t$, where $\delta \eta$ and $\delta  \tau$ are themselves continuous functions of $\delta A$ since both $\eta$ and $\tau$ depend continuously on $\kappa$ which in turn depends continuously on $\delta A$, around an invertible $A$. The theorem follows by letting $\delta A \rightarrow 0$
and using the fact that $\lim_{ \substack{\delta A \rightarrow 0\\ \kappa_A > 1}} F(A + \delta A) = F(A)$.
\end{proof}

The next result complements Theorem \ref{explicit_solution} by showing
that the rate bound in equation \eqref{tau_sol} is
the smallest one can get from the feasibility problem in \thref{jordan_theo}.
\begin{theorem}\label{optimal_solution}
If $0 < \alpha < 2$, $\rho_0 > 0$ and $\kappa \geq 1$,
then the smallest $\tau$ for which one can find a feasible point
of \eqref{semidefinite} is given by \eqref{tau_sol}.
\end{theorem}
\begin{proof}
The proof will follow by contradiction. Our counterexample follows \cite{Jordan} and \cite{Ghadimi}.
Assume that for some $0 < \alpha < 2$, $\rho_0 > 0$ and $\kappa \geq 1$
it is possible to find a feasible solution with $\tau < \nu = 1 - \frac{\alpha}{1 + \chi(\rho_0)\sqrt{\kappa}}$. Then, if we use ADMM with this $\alpha$
and a $\rho = \rho_0 \sqrt{\hat{m}\hat{L}}$ to solve any optimization problem
with this same value of $\kappa$ and satisfying Assumption~\ref{assumption} 
we have by \thref{jordan_theo} that
$\| \varphi_t - \varphi_* \| \leq C \tau^t$, where
$\tau < \nu$ and $C>0$ is some constant.
In particular, if $\rho_0 \geq 1$, this bound on the error rate must hold if we try to solve a problem where
$f(x) = \tfrac{1}{2}x^T Q x$
and $g(z) = 0$, with
$Q = \text{diag}([m,L]) \in \mathbb{R}^{2\times2}$, $A=I$, $B=-I$, and $c=0$.
Note that for this problem $\kappa_A = 1$, $\kappa = \kappa_f = L/m$, $\hat{m}  = m$ and $\hat{L} = L$.

Applying \algoref{ADMM} to this problem yields
\begin{equation}
z_{t+1} = \left(I - \alpha (Q + I \rho)^{-1} Q \right) z_t.
\end{equation}
If $z_{t=0}$ is in the direction of the smallest eigenvalue
of $Q$, the error rate for $z_t$ is
%
%
\begin{equation}
1 - \frac{\alpha}{1 + \rho m^{-1}} = 1 - \frac{\alpha}{1 + \rho_0 \sqrt{\kappa}},
\end{equation}
%
%
where in the second equality we replaced \eqref{parameters}.
But this means that the error rate for $\| \varphi_t - \varphi_* \|$ cannot be bounded by
$\tau < \nu$ for $\rho_0 \geq 1$, which contradicts our original assumption.

The proof when $\rho_0 < 1$ is similar. 
We apply ADMM to the same problem as above but now with $A = \rho_0 I$ and the rest the same. Note that for this modified problem $\kappa_A = 1$, $\kappa = \kappa_f = L/m$, $\hat{m}  = m/\rho_0^2$, $\hat{L} = L\rho_0^2$ and the $\rho$ we choose for ADMM is now $\rho = \sqrt{L m} / \rho_0$ (while before it was
 $\rho = \rho_0 \sqrt{L m}$).
\end{proof}


Now we compare the rate bound of 
ADMM with the rate bound of gradient descent (GD) when
we solve problem \eqref{minimize} with $B = I$. In what follows we
use $\tau_{\text{ADMM}}$ and $\tau_{\text{GD}}$ when talking about
rates of convergence for ADMM and GD, respectively.

Before we state our result let us discuss how  GD behaves when
we use it to solve this problem. To solve problem \eqref{minimize}
using GD with $B = I$ we reduce the problem to an unconstrained 
formulation by applying GD to the
function $F(z) = \tilde{f}(z) + g(z)$ where  $\tilde{f}(z) =  f(A^{-1}(c - z))$.
We now assume that $F \in S_{p}(m_F, L_F)$ for some $0 < m_F \leq L_F < \infty$. The work of \cite{Nesterovnotes} gives an optimally tuned rate bound
for GD when applied to any objective function in $S_p(m_F,L_F)$. This rate is
$1 - \frac{2}{1 + \kappa_F}$ where $\kappa_F = L_F / m_F$. It is easy to see that,
among all general bounds that only depend on $\kappa_F$, it is not
possible to get a function smaller than this. Indeed, if the objective function is
$x^T \text{diag}([m_F,L_F]) x$ then the rate of convergence of GD with step size $\beta$ is given by the spectral radius of the matrix
$I - \beta \text{diag}(\{m_F,L_F\})$ which is $\max\{ |1- \beta L_F|,|1- \beta m_F|\}$
and which in turn has minimum value $1 - \frac{2}{1 + \kappa_F}$ for $\beta = 2/ (L_F + m_F)$.
If $\mathcal{P}(\kappa_F)$ is the family of this unconstrained formulation of problem \eqref{minimize} with $B = I$ and $L_F/m_F = \kappa_F$, then we can summarize what we describe above as
%
%
\begin{equation}\label{GD_best_bound}
\inf_{\beta} \sup_{\mathcal{P}(\kappa_F)} \tau_{\text{GD}} = 1 - \frac{2}{1 + \kappa_F}.
\end{equation}
%
%

In a similar way, if $\mathcal{P}(\kappa)$ is the family of problems
of the form \eqref{minimize} with $B = I$, to be solved using Algorithm \ref{ADMM} under Assumption \ref{assumption}, where $f \in S_p(m,L)$ and $\kappa = L/m$, then Corollary \ref{upper_bound_admm} and the counterexample in the proof of Theorem \ref{optimal_solution} give us that
\vspace{-0.25cm}
\begin{align}\label{ADMM_best_bound}
&\inf_{\alpha, \rho_0} \sup_{\mathcal{P}(\kappa)} \tau_{\text{ADMM}}
\leq \inf_{\alpha>2, \rho_0} \sup_{\mathcal{P}(\kappa)} \tau_{\text{ADMM}} = 
1 - \frac{2}{1 + \sqrt{\kappa}},
\end{align}
%
%
where the last equality is obtained by setting $\alpha =2$ and $\rho_0 = 1$
in equation \eqref{tau_sol}.

The next theorem shows that the optimally tuned ADMM for worse-case
problems has faster convergence rate than the optimally tuned GD for worse-case problems.
\begin{theorem}
Let $\mathcal{P}(\kappa_F,\kappa)$ be the family of problems
\eqref{minimize} with $B=I$ and under Assumption \ref{assumption} such that $f \in S_p(m,L)$ with $L/m = \kappa$ and $F \in S_p(m_F,L_F)$ with $L_F/m_F = \kappa_F$, then
\vspace{-0.1cm}
\begin{equation}
\tau^*_{\text{ADMM}} \equiv \inf_{\alpha, \rho_0} \sup_{\mathcal{P}(\kappa_F,\kappa)} \tau_{\text{ADMM}}
\leq \tau^*_{\text{GD}} \equiv \inf_{\beta} \sup_{\mathcal{P}(\kappa_F,\kappa)} \tau_{\text{GD}}.
\end{equation}
More specifically,
%
\begin{equation}\label{eq:more_spec_theorem_ADMM_VS_GD}
\tau^*_{\text{GD}} \geq \dfrac{2\tau^\star_{ADMM}}{1+(\tau^\star_{ADMM})^2}.
\end{equation}
\end{theorem}
\begin{proof}
First notice that \eqref{GD_best_bound} still holds if $\mathcal{P}(\kappa_F)$ is replaced by $\mathcal{P}(\kappa_F,\kappa)$ since the objective function used in the example given above \eqref{GD_best_bound} is also in $\mathcal{P}(\kappa_F,\kappa)$.

Second notice that, since $f \in S_p(m,L)$  and $A$ is non-singular, we have 
$\tilde{f} \in S_p(m_{\tilde{f}},L_{\tilde{f}})$ for some
\mbox{$0 < m_{\tilde{f}} \leq L_{\tilde{f}} < \infty$}. Thus, since $F = \tilde{f} + g \in S_p(m_F,L_F)$, we have $g \in S_p(m_g,L_g)$ for some
$0 \leq m_g \leq L_g < \infty$ (it might be that $m_g = 0$, i.e. $g$ might not be strictly convex). Notice in addition that, without loss of generality, we can  assume that $L_F \geq L_{\tilde{f}} + m_g$, $m_F \leq m_{\tilde{f}} + m_g$, $L_{\tilde{f}} = (\sigma_1(A^{-1}))^2 L_f$ and $m_{\tilde{f}} = (\sigma_p(A^{-1}))^2 m_f$.
Therefore, if $F \in S_p(m_F, L_F)$ and $f \in S_p(m,L)$ then 
\begin{multline}
\kappa_F = \frac{L_F}{m_K} \geq \frac{L_{\tilde{f}} + m_g}{m_{\tilde{f}} + m_g}
\geq \frac{L_{\tilde{f}}}{m_{\tilde{f}} } = \frac{L_{{f}} (\sigma_1(A^{-1}))^2}{m_{{f}} (\sigma_p(A^{-1}))^2} \\
= \kappa_f (\kappa_{A^{-1}})^2 = \kappa_f (\kappa_{A})^2 = \kappa.
\end{multline}

Finally, using the fact $\kappa_F \geq \kappa$ and equations 
\eqref{GD_best_bound} and \eqref{ADMM_best_bound}, we can write
$\inf_{\alpha, \rho_0} \sup_{\mathcal{P}(\kappa_F,\kappa)} \tau_{\text{ADMM}} \leq \inf_{\alpha, \rho_0} \sup_{\mathcal{P}(\kappa)} \tau_{\text{ADMM}}\leq
1 - \frac{2}{1 + \sqrt{\kappa}} 
\leq 1 - \frac{2}{1 + \kappa} \leq 1 - \frac{2}{1 + \kappa_F} = \inf_{\beta}
\sup_{\mathcal{P}(\kappa_F,\kappa)} \tau_{\text{GD}}$.
Thus \eqref{eq:more_spec_theorem_ADMM_VS_GD} follows from the inequalities
$\kappa_F \geq \kappa$ and  
\mbox{$1 - \frac{2}{1 + \sqrt{\kappa}} \leq 1 - \frac{2}{1 + \kappa_F} $}.
\end{proof}
%

\section{Numerical Results}

We now compare numerical solutions to the SDP in \thref{jordan_theo} 
with the exact formulas from \thref{explicit_solution}. 
The numerical procedure was implemented in MATLAB using CVX and a binary search
to find the minimal $\tau$ such that \eqref{semidefinite} is feasible.
This is exactly the same procedure described in \cite{Jordan} and it works
because the maximum eigenvalue of \eqref{semidefinite} decrease monotonically with $\tau$.
\figref{tau_numerical} shows the rate bound $\tau$ against $\kappa$ 
for several choices of parameters $(\alpha,\,\rho_0)$. 
The dots correspond to the numerical solutions
and the solid lines correspond to the exact formula \eqref{tau_sol}. 
\figref{lambdas_numerical} compare
the numerical values of $\lambda_1$ (circles) and $\lambda_2$ (squares)
with the formulas \eqref{l1_sol} and \eqref{l2_sol} (solid lines).
There is a perfect agreement between \mbox{\eqref{P_sol}--\eqref{tau_sol}} and 
the numerical results, which strongly
support \thref{explicit_solution} and \thref{optimal_solution}.

The range \mbox{$0<\alpha<1$} give worse convergence
rates compared to $1 \le \alpha < 2$.
The best rate bound is 
attained with $\rho_0=1$, or equivalently $\rho=\sqrt{\hat{m}\hat{L}}$, and
\mbox{$\alpha = 2$}. This is also evident from \eqref{tau_sol}.
Note, however, that \eqref{eta_rate} diverges when $\alpha\to2$
so although the optimal rate bound, in the asymptotical sense, is $1 - \frac{2}{1 + \chi(\rho_0)\sqrt{\kappa}}$, bound \eqref{convergence} suggests that in a
practical setting with a maximum number of iterations it might be better to
choose $\alpha < 2$.

\begin{figure}
\centering
\includegraphics[width=.8\linewidth]{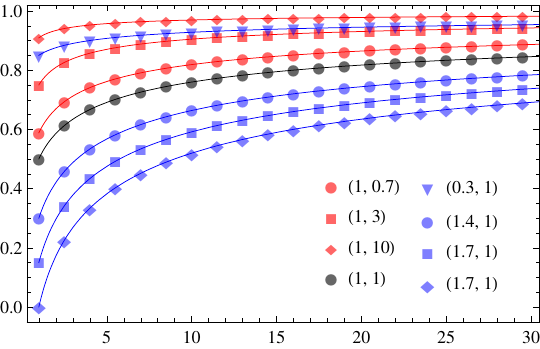}
\put(-100,-7){\small $\kappa$}
\put(-210,70){\small $\tau$}
\put(-60,68){\footnotesize $(\alpha,\rho_0)$}
\caption{ \label{tau_numerical}
Plot of $\tau$ versus $\kappa$ for different values of
parameters $(\alpha,\rho_0)$, as indicated in the legend. 
The dots correspond to the numerical solution to
\eqref{semidefinite} while the solid curves are the exact formula 
\eqref{tau_sol}.
The best choice of parameters are $\rho_0=1$ and $\alpha=2$. The convergence
rate is improved with the choice $1 \le \alpha < 2$ compared to $0<\alpha<1$. 
}
\end{figure}
\begin{figure}
\centering
\includegraphics[width=.8\linewidth]{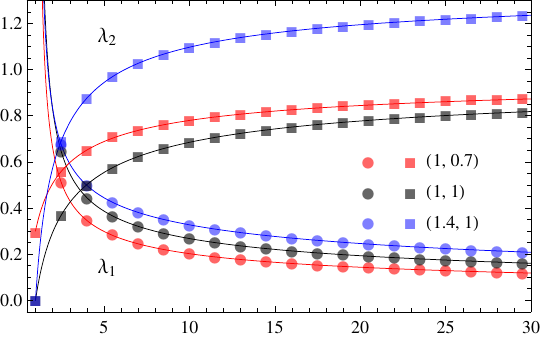}
\put(-100,-7){\small $\kappa$}
\put(-226,70){\small $\lambda_1,\lambda_2$}
\put(-60,72){\footnotesize $(\alpha,\rho_0)$}
\caption{
We show $\lambda_1$ (circles) and $\lambda_2$ (squares) 
verus $\kappa$ for some of the choices of parameters
$(\alpha, \rho_0)$ in \figref{tau_numerical}. Note the exact match 
of numerical results with formulas \eqref{l1_sol} and \eqref{l2_sol} 
(solid lines).
}
\label{lambdas_numerical}
\end{figure}

Corollary \ref{upper_bound_admm} is valid only
for $0<\alpha <  2$ (for $\alpha > 2$, \eqref{tau_sol} can 
assume negative values). However,  
\thref{jordan_theo} does not impose any
restriction on $\alpha$ and holds even for $\alpha > 2$
\cite{Jordan}. 
To explore the range $\alpha > 2$
we numerically solve \eqref{semidefinite} as  shown in \figref{alpha_bigger}. 
\begin{figure}
\centering
\includegraphics[width=.8\linewidth]{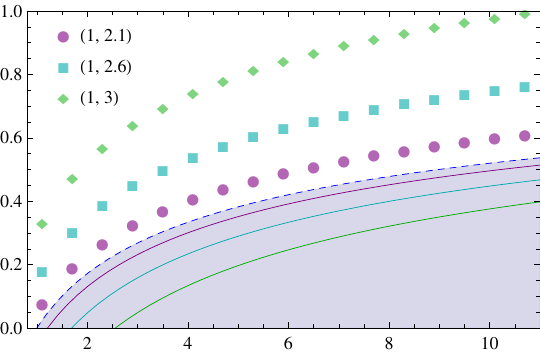}
\put(-100,-7){\small $\kappa$}
\put(-210,70){\small $\tau$}
\put(-143,105){\footnotesize $(\alpha,\rho_0)$}
\caption{Plot of $\tau$ versus $\kappa$ for some values of $(\alpha,\rho_0)$
with $\alpha>2$ and $\rho_0=1$.
The dashed blue line corresponds to $\alpha=2$ in
formula \eqref{tau_sol}. The shaded region contains curves given by
\eqref{tau_sol} for values of $\alpha$ not allowed in \thref{explicit_solution}.
These numerical solutions with $\alpha > 2$ had to be restricted to a range
$1 < \kappa \lesssim 11$. Moreover, notice that $\alpha > 2$ does not produce
better convergence rates than $1\le\alpha<2$ through \eqref{tau_sol}, which 
is valid for any $\kappa > 1$.}
\label{alpha_bigger}
\end{figure}
The dots correspond to the numerical solutions.
The dashed blue line corresponds to \eqref{tau_sol} with
$\alpha=2$, and it is the boundary of the shaded region in which
\eqref{tau_sol} can have negative values and is no longer
valid. Although \thref{explicit_solution} does not hold for $\alpha>2$, 
we deliberately 
included the solid lines representing \eqref{tau_sol} inside
this region.  Obviously, these curves do not match the numerical results. 

The first important remark is that, for a given $\alpha > 2$, we were  
unable to numerically find solutions for arbitrary $\kappa \geq 1$. 
For instance, for $\alpha=2.6$ we can only stay roughly on
the interval $1<\kappa \lesssim11$. 
The same behavior occurs for any \mbox{$\alpha > 2$}, and the range of
$\kappa$ becomes narrower as $\alpha$ increases.
From the picture one can notice that $\tau = 1$ is actually attained with 
\emph{finite} $\kappa$, while for \eqref{tau_sol} this never happens;
it rather approaches $\tau\to 1^{-}$ as $\kappa \to \infty$.
Therefore, although it is feasible to solve \eqref{semidefinite}
with $\alpha > 2$, the solutions will be constrained to a small range
of $\kappa$. The next question would be if \thref{jordan_theo} for
$\alpha > 2$ could possibly
give a better rate bound than Corollary \ref{upper_bound_admm} with $1 \le \alpha < 2$. We 
can see from the picture that this is probably not the case.
We conclude that, as far as solutions to \eqref{semidefinite} are considered, 
there is no advantage in considering $\alpha > 2$ compared to 
\eqref{tau_sol} with $1 \le \alpha < 2$, and 
which holds for arbitrary $\kappa > 1$.
It is an interesting problem to determine if proof techniques other than
\cite{Lessard} can lead to good rate bounds for $\alpha > 2$.

\section{Conclusion}

We introduced a new explicit rate bound for the
entire family of over-relaxed ADMM.
Our bound is the first of its kind and improves 
on \cite{Jordan} and \cite{wei_almost_bound}.
In particular, the only explicit bound in \cite{Jordan} is a special
case of our general explicit formula when $\kappa$ is large.
We also show that our bound is the best one can extract from the
integral quadratic constraints framework of \cite{Lessard}. 
In \cite{nesterov2004introductory} we find that $1 - 2 / (1 + \sqrt{\kappa})$
bounds the convergence rate of any first order method on $S(m,L)$, $\kappa =
m/L$, so we have also shown that ADMM with $\alpha \rightarrow 2$ is close to being optimal on $S(m,L)$.

Although our analysis assumes that $f$ is strongly convex,
we can use a very-slightly modified ADMM algorithm to solve problem
\ref{minimize} when $f$ is weakly convex using an idea of Elad Hazan
explained in \cite{Lessard} (Section 5.4).






%

\end{document}